\documentclass{article}
\usepackage{iclr2027_conference,times}
\usepackage{amsmath,amssymb,amsthm,mathtools}
\usepackage{graphicx,booktabs,array,multirow}
\usepackage{wrapfig,placeins,pifont}
\usepackage{microtype,xcolor,url,hyperref}

\hypersetup{colorlinks=true,linkcolor=black,citecolor=black,urlcolor=black,
  pdftitle={DoSA: Dots Sparse Attention},
  pdfauthor={Fanxu Meng, Taosong Fang}}
\title{DoSA: Dots Sparse Attention}
\author{%
\textbf{Fanxu Meng}\thanks{Equal contribution. Author order was determined by a coin toss.}\\
Dots Studio, Xiaohongshu Inc.\\
\href{mailto:mengfanxu@xiaohongshu.com}{\texttt{mengfanxu@xiaohongshu.com}}
\And
\textbf{Taosong Fang}\footnotemark[1]\thanks{Corresponding author.}\\
Dots Studio, Xiaohongshu Inc.\\
\href{mailto:fangtaosong@xiaohongshu.com}{\texttt{fangtaosong@xiaohongshu.com}}
}
\iclrfinalcopy 
\makeatletter
\renewcommand{\@maketitle}{\vbox{\hsize\textwidth\centering
  \def\@makefnmark{\hbox{\@textsuperscript{\normalfont\@thefnmark}}}%
  {\LARGE\scshape\@title\par}
  \def\And{\end{tabular}\hspace{2em}%
    \begin{tabular}[t]{c}\rule{\z@}{24pt}\ignorespaces}%
  \begin{tabular}[t]{c}\rule{\z@}{24pt}\@author\end{tabular}\par
  \vskip 0.6em
  {\normalfont\small\url{https://github.com/llm-sys/flash-indexer}\par}
  \vskip 0.3in minus 0.1in
}}
\makeatother
\newcommand{\method}{\textsc{DoSA}}
\newcommand{\relu}[1]{\left[#1\right]_{+}}
\newcommand{\norm}[1]{\left\lVert#1\right\rVert_2}
\newcommand{\topk}{\operatorname{TopK}}
\newcommand{\q}{\mathbf q}
\newcommand{\kv}{\mathbf k}
\newcommand{\RR}{\mathbb R}
\newcommand{\sref}{s^{\mathrm{ref}}}
\newtheorem{proposition}{Proposition}
\newtheorem{theorem}{Theorem}

\newcounter{algorithm}
\newcommand{\algline}[2]{\makebox[1.6em][r]{\scriptsize #1}\quad #2\par}
\begin{document}
\maketitle
\lhead{} 
\begin{abstract}
Sparse attention visits few keys, yet finding them can still require scoring every key with every indexer head. We introduce \method{} (\emph{Dots Sparse Attention}), a lossless method for making this computation conditional while preserving exact token selection. Its central observation is that Top-$K$ selection requires sufficient evidence to reject a candidate, not necessarily its full score. A seed set supplies a lower bound on the winning threshold. Nonnegative head projections and joint residual bounds upper-bound positive contributions, while activation-informed aggregation bounds the magnitude subtracted by negative heads. The certificate directs full scoring toward candidates that remain eligible; exact-score filtering further reduces writeback and selection work. We establish the certificate and its exact-selection conditions, and analyze how geometry and selection margins govern its effectiveness. Our GPU implementation accelerates the complete Indexer by up to $2.74\times$ over SGLang and $1.69\times$ over LiteTopK across English, Chinese, and code workloads. End-to-end evaluation on GLM-5.2 and DeepSeek-V3.2 achieves up to $1.94\times$ and $1.50\times$ time-to-first-token speedups over the two baselines, respectively.
\end{abstract}

\section{Introduction}
Conditional computation can scale model capacity while activating only part of the model for each input. Mixture-of-experts models apply this principle in feed-forward networks \citep{switch}; attention offers a complementary opportunity as contexts grow. Sparse attention visits few keys, yet identifying them can still require a dense scan. DeepSeek Sparse Attention (DSA) selects tokens with a multihead indexer followed by Top-$K$ \citep{deepseek32}; each score requires dot products, ReLU activations, and a signed weighted reduction. This leaves a dense scoring bottleneck inside sparse attention, limiting the MoE-style separation of capacity and active computation. We take a step toward scalable attention by making token selection conditional while preserving the original routing rule.

Figure~\ref{fig:indexer_cost} illustrates the resulting bottleneck. At approximately one million tokens, the indexer accounts for 73.0\% of prefill GPU time in GLM-5.2 and 90.8\% in DeepSeek-V3.2, while its shares of logical matrix FLOPs are 47.4\% and 83.2\%, respectively. As context grows, choosing where to attend can dominate the cost of sparse attention itself. The gap between time and arithmetic shares also motivates reducing selection and memory traffic alongside score computation.

\begin{figure}[t]
  \centering
  \includegraphics[width=\linewidth]{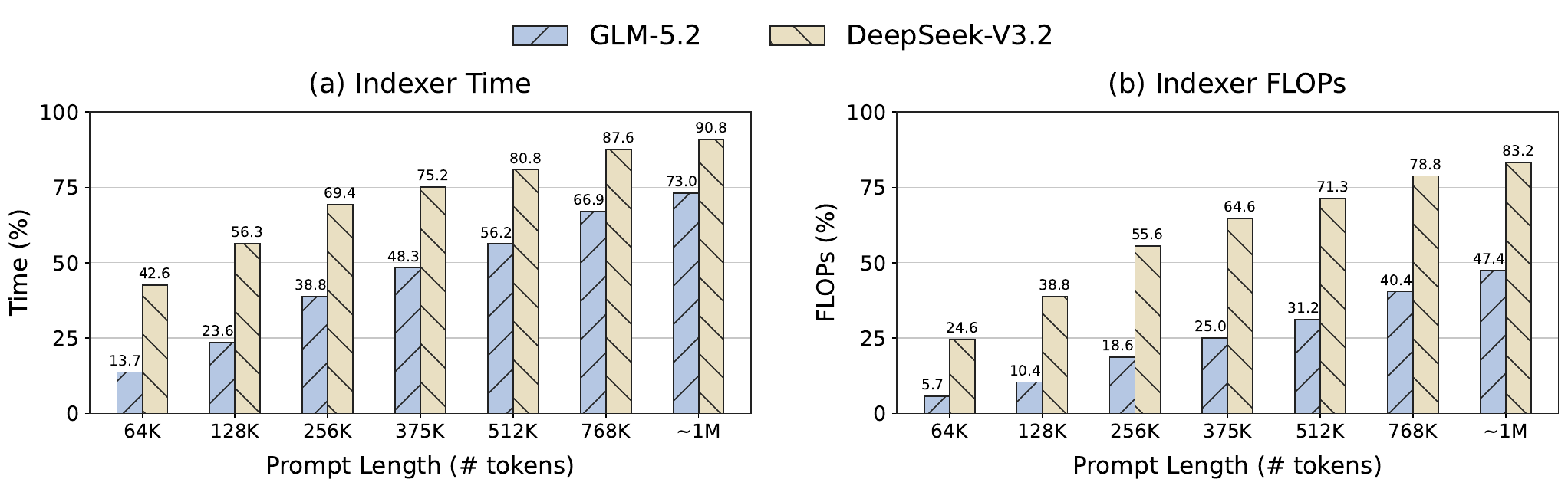}
  \caption{The indexer becomes a dominant cost of long-context prefill. Left: its fraction of prefill GPU time, including query/key preparation, scoring, and Top-$K$. Right: its fraction of theoretical matrix FLOPs. Each point averages two requests under SGLang; profiling and FLOP conventions are given in Appendix~\ref{app:intro_profile}.}
  \label{fig:indexer_cost}
\end{figure}

Recent systems accelerate selection: Guess-Verify-Refine (GVR) exploits temporal correlation in decoding \citep{gvr}; DeepSelect streams materialized scores through a bounded buffer \citep{deepselect}; and LiteTopK fuses scoring with candidate filtering to avoid dense score writeback \citep{litetopk}. These paths still require all-head evaluation for every key. MISA instead screens with fewer heads and reranks candidates with the full indexer \citep{misa}, but reranking cannot recover excluded winners. Can we \emph{avoid full-score computation while certifying that screening discards no winner?}

The principle is simple: \emph{a token cannot enter Top-$K$ if its score upper bound falls below a lower bound on the true Top-$K$ threshold.} A subset of at least $K$ visible keys supplies the latter through its $K$th largest score, following established exact-retrieval principles \citep{maxscore1995,wand2003}. The challenge is to construct a tight bound at substantially lower cost than the score itself. For the signed ReLU indexer, dropping positive-weight heads is unsafe; dropping negative-weight heads is safe but can leave loose bounds. A useful certificate must exploit the geometry and activation structure of the heads.

\method{}, a lossless optimization, constructs token-specific bounds by evaluating selected heads and bounding the rest. Nonnegative projections reuse retained dot products, and joint residual bounds exploit dependence among omitted positive heads. An activation-informed aggregate lower-bounds the magnitude subtracted by omitted negative heads. These query-side quantities are reused across keys. Figure~\ref{fig:pipeline} shows how the certificate leads to three complementary savings: fewer score computations, fewer score writes, and less Top-$K$ work. Survivors receive full original scores, and only scores meeting the seed-derived threshold are written for final selection.

Our contributions are a certificate tailored to the signed DSA score, an analysis connecting its slack to selection margins and head budgets, and a GPU realization that turns certified rejection into reduced computation and traffic. Across English, Chinese, and code workloads, \method{} outperforms both baselines in all 60 microbenchmark configurations, with up to $2.74\times$ complete-Indexer speedup over SGLang and $1.69\times$ over LiteTopK. At approximately 1M tokens, it achieves $1.86\times$ and $1.94\times$ time-to-first-token (TTFT) speedups over SGLang on GLM-5.2 and DeepSeek-V3.2, respectively, nearly halving end-to-end latency.

\section{Method}
\label{sec:method}

\begin{figure}[t]
  \centering
  \includegraphics[width=0.9\linewidth]{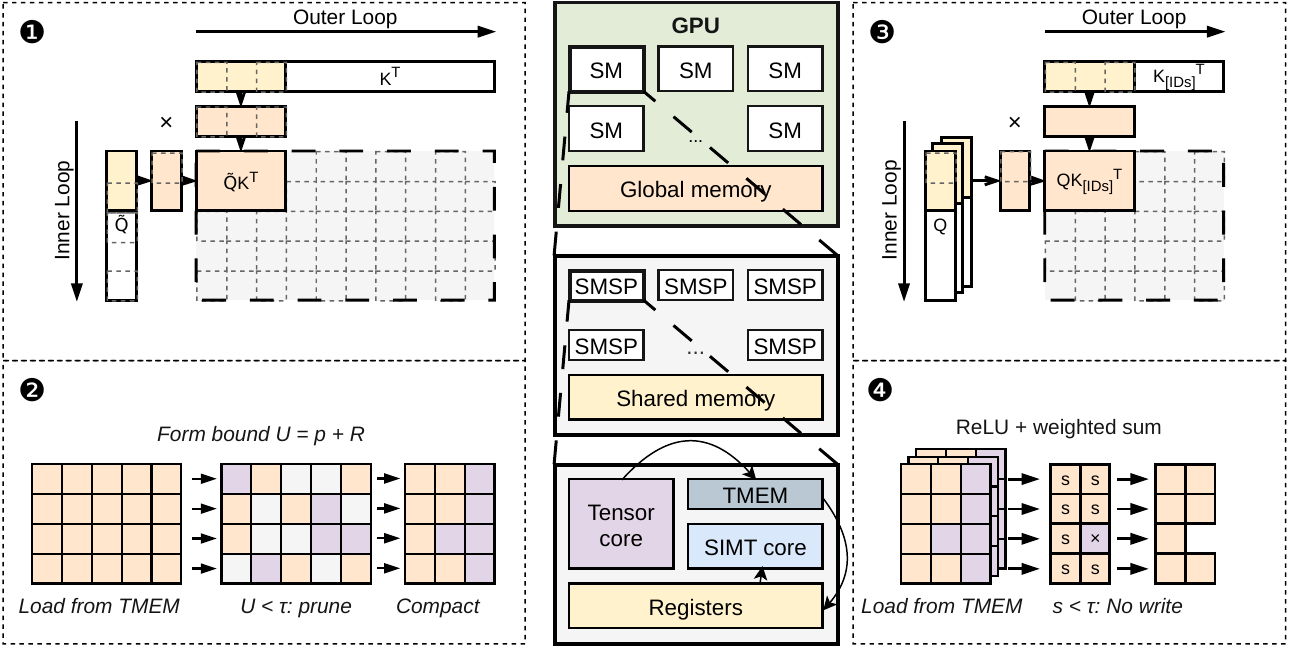}
  \caption{\method{} execution overview, read in order \ding{182}--\ding{185}: \ding{182} screening dot products; \ding{183} bound evaluation, pruning, and candidate compaction; \ding{184} full-head dot products for survivors; \ding{185} original-score reduction and sparse writeback. Tensor Cores execute stages~\ding{182} and~\ding{184}; SIMT threads execute stages~\ding{183} and~\ding{185} using accumulators from Tensor Memory (TMEM), with gray and purple data tiles indicating avoided computation and writeback, respectively. The center sketches GPU resources; seed initialization and final Top-$K$ selection are omitted.}
  \label{fig:pipeline}
\end{figure}

\subsection{Problem formulation}
\label{sec:problem}
For a fixed query with $H$ indexer heads, let $\q_h\in\RR^D$ be the $h$th query vector, $w_h\in\RR$ its weight, and $\kv_i\in\RR^D$ the key for token $i$. Write $[x]_+=\max(x,0)$. The indexer score is
\begin{equation}
  s_i=\sum_{h=1}^{H} w_h\relu{\q_h^\top\kv_i},\qquad i\in\Omega,
  \label{eq:score}
\end{equation}
where $\Omega$ contains only visible keys.

For $|\Omega|\ge K$, our goal is to recover the $K$ highest-scoring visible keys while avoiding full-head scoring for keys that cannot enter this set.

\subsection{Seed threshold and retained heads}
\label{sec:seed}
Choose distinct visible keys $\mathcal S\subseteq\Omega$, with $|\mathcal S|\ge K$, and set $\tau=\tau_K(\mathcal S)$ to their $K$th largest full score. The subset supplies $K$ witnesses:
\begin{equation}
  \tau_K(\mathcal S)\le\tau_K(\Omega),\qquad
  s_i\le U_i<\tau_K(\mathcal S)\ \Longrightarrow\ i\notin\topk(\Omega).
  \label{eq:prune}
\end{equation}
Thus $\tau_K(\mathcal S)$ is a lower bound on the true Top-$K$ threshold, and $U_i$ is an upper bound on the token score: their strict separation certifies rejection. A poor seed weakens filtering, and additional exact scores can only raise the threshold lower bound.

Bounding every head can leave all candidates above the threshold. We therefore directly evaluate selected heads whose contributions are difficult to bound tightly.

Let $R$ be a set of $1\le r<H$ retained heads. Define $\mathcal H_+=\{h:w_h>0\}$ and $\mathcal H_-=\{h:w_h<0\}$. If $|\mathcal H_+|\le r$, retain all positive heads and fill remaining positions with negative heads prioritized by $|w_h|\norm{\q_h}\min(p_h,1-p_h)$, where $p_h$ is their activation frequency on sampled seed keys. This favors strong heads whose activation is difficult to predict. Otherwise, retain the $r$ strongest positive heads by $w_h\norm{\q_h}$. Zero-weight heads need no contribution. Write $P=\mathcal H_+\setminus R$ and $N=\mathcal H_-\setminus R$. The retained contribution
\begin{equation}
  S_R(\kv_i)=\sum_{j\in R}w_j\relu{\q_j^\top\kv_i}
\end{equation}
is evaluated directly. The omitted heads in $P$ and $N$ are not evaluated individually during screening; their contributions are bounded in the next two subsections. The proof holds for any $R$; its choice affects efficiency and tightness.

\subsection{Upper-bounding omitted positive heads}
\label{sec:positive}
For each $h\in P$, choose a retained anchor and a nonnegative projection:
\begin{equation}
 j_h^\star=\arg\max_{j\in R}\frac{\relu{\q_h^\top\q_j}}{\norm{\q_j}},\quad
 \beta_h=\frac{\relu{\q_h^\top\q_{j_h^\star}}}{\norm{\q_{j_h^\star}}^2},\quad
 \mathbf d_h=\q_h-\beta_h\q_{j_h^\star}.
 \label{eq:projection}
\end{equation}
A zero-norm anchor has matching score and coefficient zero. ReLU subadditivity, positive homogeneity, and Cauchy--Schwarz give
\begin{equation}
 \relu{\q_h^\top\kv_i}
 \le \beta_h\relu{\q_{j_h^\star}^\top\kv_i}
       +\norm{\mathbf d_h}\norm{\kv_i}.
 \label{eq:positive}
\end{equation}
The projected part reuses an already computed dot product. The nonnegative projection minimizes the residual norm, yielding the smallest uniform norm allowance among single-anchor bounds of this form (Appendix~\ref{app:projection_optimal}). Define
\begin{equation}
 \widetilde w_j=w_j+\sum_{\substack{h\in P\\j_h^\star=j}}w_h\beta_h,
 \qquad E_P=\sum_{h\in P}w_h\norm{\mathbf d_h}.
 \label{eq:fold}
\end{equation}
Then retained and omitted positive terms together are bounded by
$\sum_{j\in R}\widetilde w_j[\q_j^\top\kv_i]_++E_P\norm{\kv_i}$.
The folded weights belong only to the screening plan; they do not modify model parameters. If $P$ is empty, $E_P=0$ and $\widetilde w_j=w_j$.

The omitted residuals also share the same key. Let $\mathbf u_h=w_h\mathbf d_h$ for $h\in P$. A joint allowance is
\begin{equation}
 E_P^{\rm joint}=\tfrac12\left(
 \left\lVert\sum_{h\in P}\mathbf u_h\right\rVert_2+
 \sqrt{\sum_{h,\ell\in P}|\mathbf u_h^\top\mathbf u_\ell|}\right)
 \le E_P.
 \label{eq:positive_joint}
\end{equation}
It bounds $\sum_{h\in P}[\mathbf u_h^\top\kv_i]_+$ by $E_P^{\rm joint}\norm{\kv_i}$ and can therefore replace $E_P$. The improvement exploits residual geometry without additional key reads. Appendix~\ref{app:positive_joint} proves this statement and its grouped form.

\subsection{Lower-bounding omitted negative contributions}
\label{sec:negative}
Negative heads subtract from the score. To construct an upper bound, we aggregate a subset $A\subseteq N$ and ignore the remaining heads; omitting negative contributions can only raise the bound. Define
\begin{equation}
 \mathbf a_A=\sum_{h\in A}|w_h|\q_h,\qquad
 L_N(\kv_i)=\relu{\mathbf a_A^\top\kv_i}
 \le \sum_{h\in N}|w_h|\relu{\q_h^\top\kv_i}.
 \label{eq:negative}
\end{equation}
ReLU subadditivity ensures that $L_N$ does not exceed the true amount subtracted by the omitted negative heads. Why not simply include every head in $A$? An inactive head contributes zero after its own ReLU, but its negative dot product can cancel active heads' positive dot products before the aggregate ReLU. Including such heads can therefore weaken the negative-contribution estimate and loosen the score upper bound.

We use seed activation patterns to reduce this cancellation. For sampled seed keys, record $M_i=\{h\in N:\q_h^\top\kv_i>0\}$ and let $p_h$ be the frequency of head $h$. We choose componentwise majority support, $A=\{h\in N:p_h\ge1/2\}$. This minimizes a separable upper bound on normalized negative slack; jointly, the retained-negative priority above removes the largest remaining terms of that surrogate (Appendix~\ref{app:support_surrogate}). If a new key has active set exactly $A$, the aggregate is exact. Other patterns remain safe and may give a looser bound. The frequencies guide efficiency, not correctness.

\subsection{Score bounds and candidate refinement}
\label{sec:refinement}
Combining Equations~\eqref{eq:positive}--\eqref{eq:negative} gives our score certificate:
\begin{equation}
 \boxed{U_i=\sum_{j\in R}\widetilde w_j\relu{\q_j^\top\kv_i}
           +E_P\norm{\kv_i}-\relu{\mathbf a_A^\top\kv_i},\qquad s_i\le U_i.}
 \label{eq:bound}
\end{equation}
One aggregate direction replaces arbitrarily many omitted negative heads. Its utility depends on activation agreement, while the positive allowance depends on directional similarity.

Figure~\ref{fig:pipeline} maps the certificate to four stages. Stage~\ding{182} reuses screening rows $\widetilde Q$ (retained and auxiliary) across key tiles; outer and inner loops traverse keys and query rows, respectively. Stage~\ding{183} forms $U=p+R$: $p$ combines the retained and aggregate terms in Equation~\eqref{eq:bound}, and the scalar $R=E_P\norm{\kv_i}$ is the residual allowance, distinct from the head set $R$. FP8 adds numerical guards (Appendix~\ref{app:numerical}).

Stage~\ding{183} discards keys with $U_i<\tau$ and compacts survivor IDs. Stage~\ding{184} gathers their keys and evaluates all $H$ original heads. Stage~\ding{185} applies the original ReLU and weights from Equation~\eqref{eq:score}, retaining only scores $s_i\ge\tau$ for selection with the seed records. In the data tiles of stages~\ding{183} and~\ding{185}, gray marks avoided computation and purple marks avoided writeback. Equality survives both filters; bounds never replace final scores.

With valid seed thresholds and no loss of live candidates, this procedure preserves exact Top-$K$ selection; Appendix~\ref{app:proofs} gives the complete algorithm, formal guarantee (Theorem~\ref{thm:exact}), and proof.

\section{GPU Kernel Design}
\label{sec:implementation}
\label{sec:streaming}
Certified pruning reduces arithmetic, but irregular survivor sets can limit GPU utilization and introduce scattered key accesses. During prefill, we organize \method{} into a regular screening pass followed by sparse full-head refinement (Figure~\ref{fig:kernel}). Three design choices translate the certificate into efficient execution: grouping queries for Tensor Core computation, sharing key loads during refinement, and reducing the records passed to selection.

\begin{figure}[t]
\centering
\includegraphics[width=\linewidth]{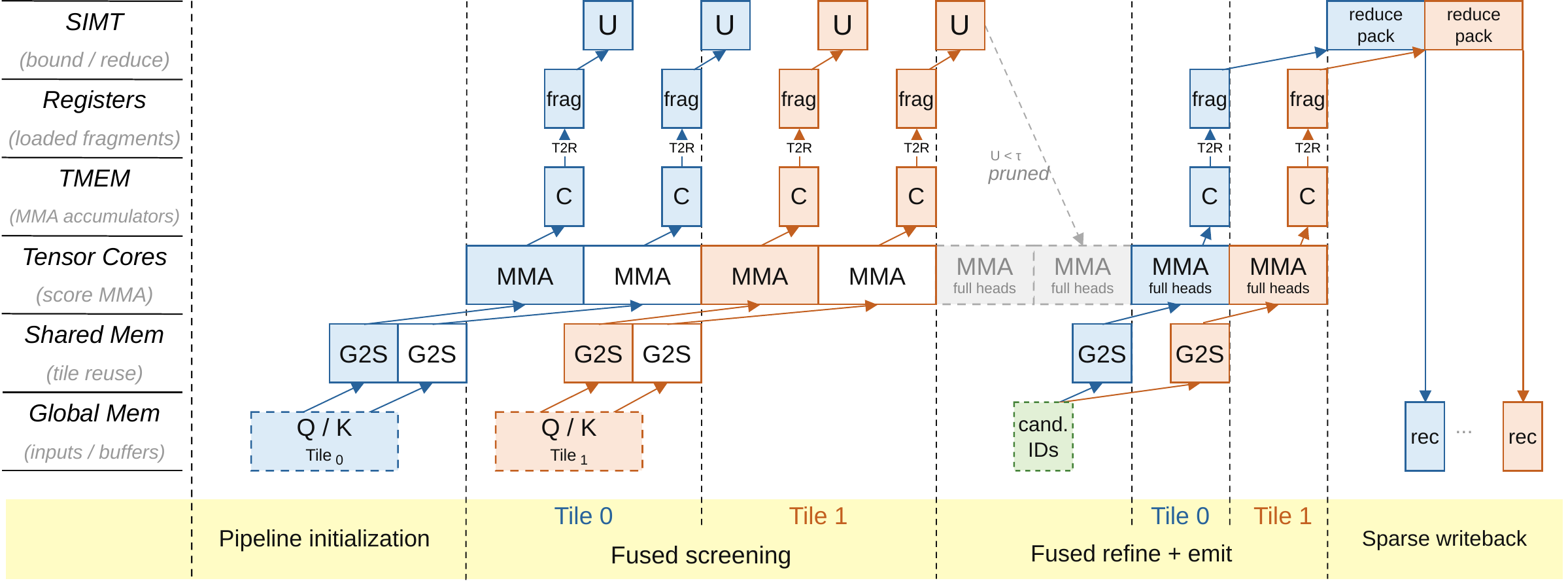}
\caption{Data flow through screening, refinement, and sparse writeback. Blue and orange illustrate successive tiles within each stage, where key loading overlaps computation. SIMT threads evaluate bounds $U$ from TMEM accumulators; gray full-head MMA blocks indicate work avoided through certified pruning. Compacted candidate IDs feed refinement, whose outputs are filtered before writeback. G2S and T2R denote global-to-shared and TMEM-to-register transfers.}
\label{fig:kernel}
\end{figure}

\paragraph{Query grouping and pipelined screening.}
We group queries with similar retained-head budgets into fixed-width tiles, enabling a small set of screening kernels to use dense Tensor Core MMA. Retained query rows, folded weights, and two FP8 auxiliary rows are prepared once and reused across key tiles. Exact scores of visible seed keys initialize each query's threshold, which remains fixed during the scan. Dedicated warps issue asynchronous key loads and MMA operations, while consumer warps evaluate bounds from TMEM accumulators. Buffered key tiles overlap loading with computation. Keys are rejected only when $U_i^{\rm num}<\tau$, and warp ballots pack the survivor decisions into compact masks for refinement.

\paragraph{Shared key access in refinement.}
Independent per-query gathers would repeatedly load keys selected by several queries. We therefore compact the survivor masks into shared and per-query refinement streams. Within a group of four queries, a key that survives at least two bounds and lies in the group's common visible range enters the shared stream; other candidates use per-query streams. A shared key load serves full-head scoring for the group, followed by the original ReLU, weights, and reduction. This organization trades some extra query--key evaluations for fewer key transfers. Each score is filtered against its query's threshold, preserving exact selection.

\paragraph{Compact writeback and exact selection.}
Refinement fuses full-score computation with threshold filtering: only scores $\sref_i\ge\tau$ produce output records. These records store compact monotone score codes with exact key IDs, reducing HBM traffic and selection input. After refinement, Top-$K$ selection retains records strictly above the boundary code and resolves candidates at the boundary using their original scores. Together, screening reduces full-head arithmetic, shared refinement reuses key data, and compact records reduce the cost of writing and selecting scores.

\paragraph{Decoding.}
With few queries per step, decoding offers limited key reuse across query groups. We retain full FP8 scoring and exploit temporal coherence to reduce selection work. Distinct keys selected in recent steps provide witnesses whose current-query scores yield a conservative lower bound on the Top-$K$ threshold. Scores at or above this bound are compacted for exact selection, which also maps the selected indices and updates the witnesses.

\setcounter{topnumber}{1}

\begin{figure}[t]
\centering
\includegraphics[width=\linewidth]{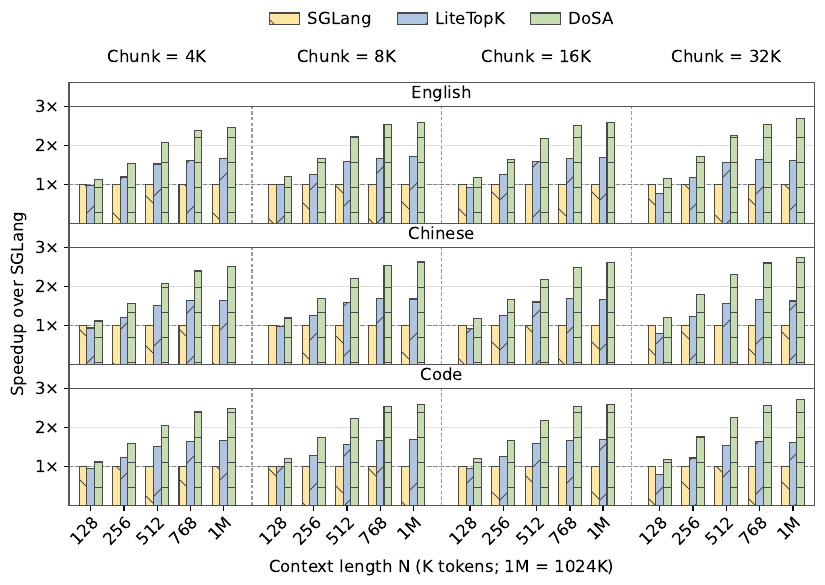}
\caption{Complete-Indexer speedup over SGLang across context lengths, input domains, and prefill chunk sizes.}
\label{fig:indexer_generalization}
\end{figure}

\begin{figure}[t]
\centering
\begin{minipage}[t]{0.25\linewidth}
\vspace{0pt}
\centering
\includegraphics[width=\linewidth]{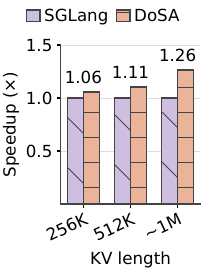}
\par
\begin{minipage}[t]{0.92\linewidth}
\caption{Decoding Indexer speedup over SGLang.}
\label{fig:decoding_indexer}
\end{minipage}
\end{minipage}%
\begin{minipage}[t]{0.5\linewidth}
\vspace{0pt}
\centering
\includegraphics[width=\linewidth]{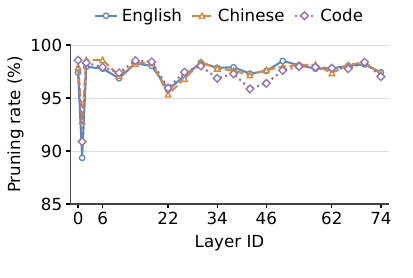}
\par
\begin{minipage}[t]{0.92\linewidth}
\caption{Certificate pruning rates across GLM-5.2 layers, averaged within each input domain.}
\label{fig:indexer_pruning}
\end{minipage}
\end{minipage}%
\begin{minipage}[t]{0.25\linewidth}
\vspace{0pt}
\centering
\includegraphics[width=\linewidth]{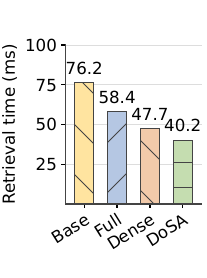}
\par
\begin{minipage}[t]{0.92\linewidth}
\caption{Screening and sparse writeback ablation.}
\label{fig:ablation}
\end{minipage}
\end{minipage}
\end{figure}

\begin{figure}[t]
\centering
\begin{minipage}[b]{0.49\linewidth}
\centering
\includegraphics[width=\linewidth]{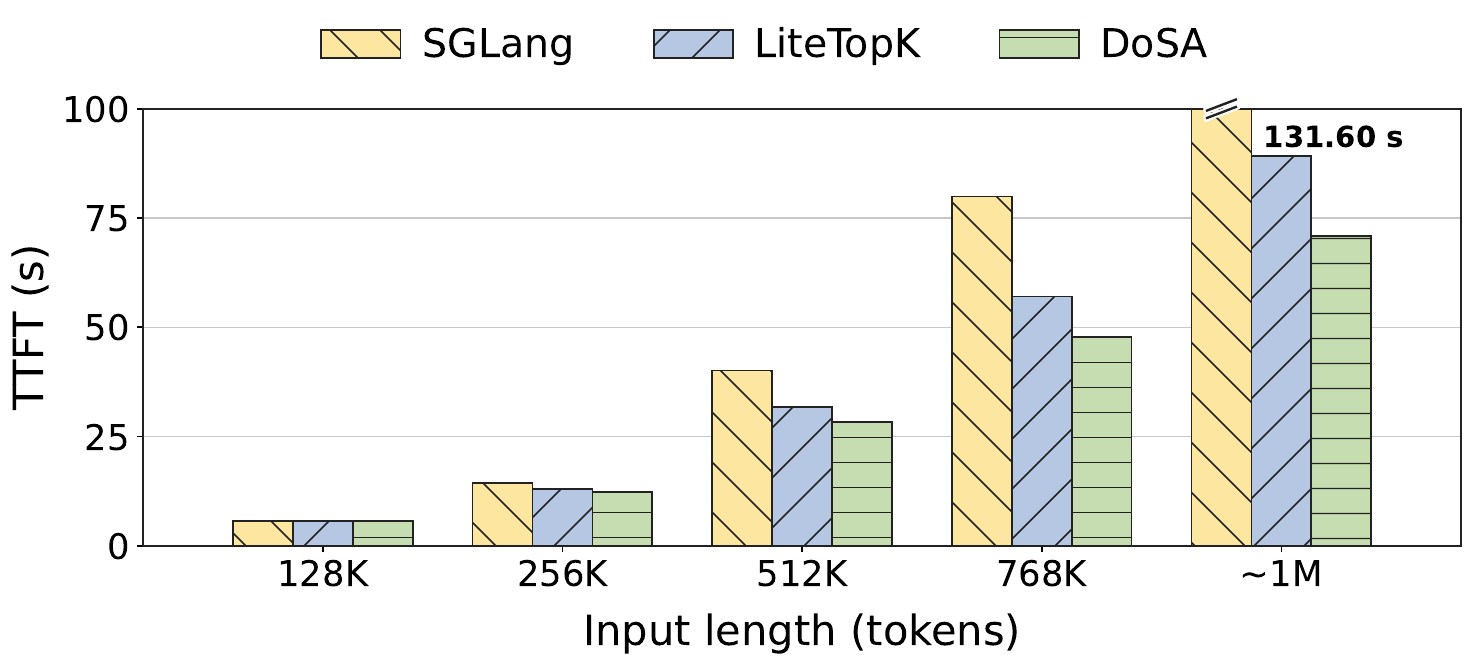}
\par\smallskip
{\small (a) GLM-5.2}
\end{minipage}\hfill
\begin{minipage}[b]{0.49\linewidth}
\centering
\includegraphics[width=\linewidth]{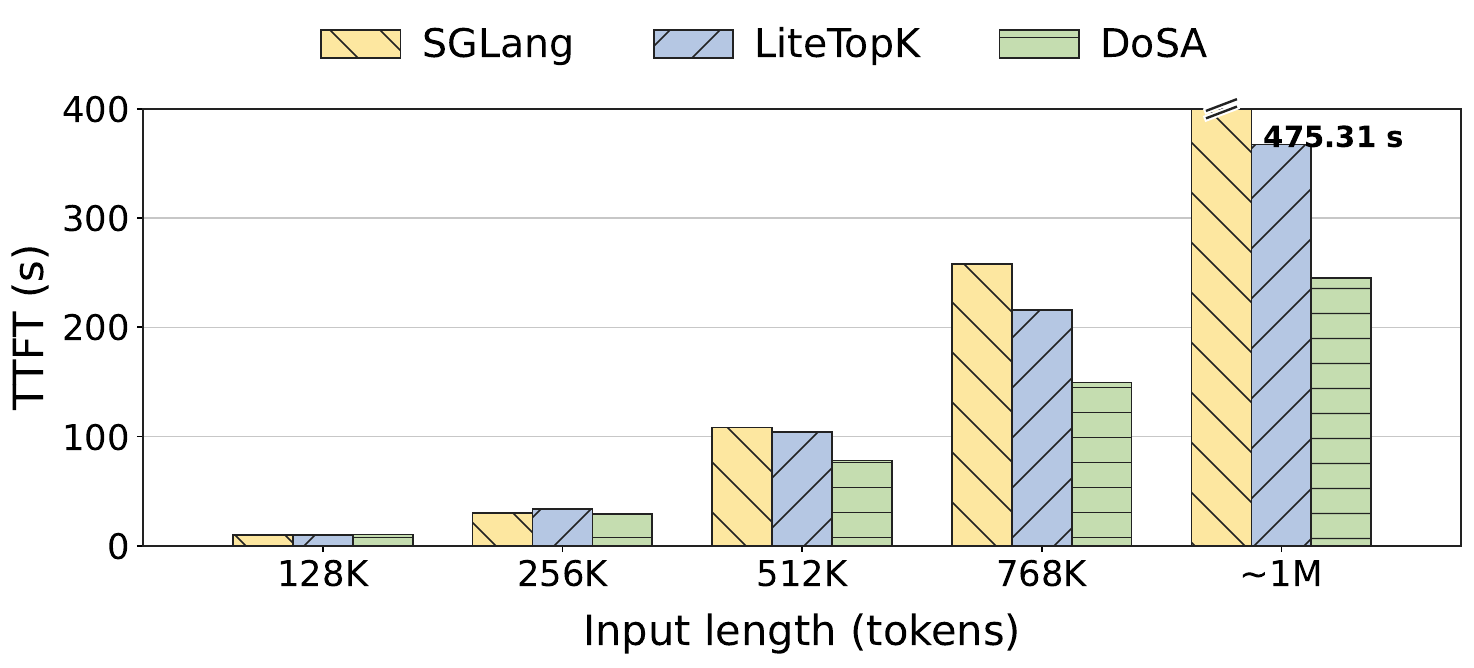}
\par\smallskip
{\small (b) DeepSeek-V3.2}
\end{minipage}
\caption{Time to first token across input lengths on (a) GLM-5.2 and (b) DeepSeek-V3.2.}
\label{fig:ttft_length}
\end{figure}

\section{Experiments}
\label{sec:experiments}
\method{} achieves up to $2.74\times$ complete-Indexer speedup and $1.94\times$ TTFT speedup over SGLang. We study how these gains vary with context length, input domain, and prefill chunk size, then isolate the contributions of screening and sparse writeback.

\subsection{Experimental setup}
\paragraph{Models and hardware.}
We evaluate GLM-5.2 and DeepSeek-V3.2 on NVIDIA GPUs. The models use 32 and 64 indexer heads, respectively, with Top-$K=2{,}048$. Microbenchmarks run GLM-5.2 on a single GPU. End-to-end inference uses eight-way tensor parallelism, with prefill chunks of 32K tokens for GLM-5.2 and 16K for DeepSeek-V3.2.

\paragraph{Workloads.}
We use English and Chinese books and code inputs from InfiniteBench~\citep{infinitebench}, supplemented with Project Gutenberg\footnote{\url{https://www.gutenberg.org/ebooks/100}; \url{https://www.gutenberg.org/ebooks/3200}} and SQLite source code\footnote{\url{https://www.sqlite.org/amalgamation.html}} for long contexts. For each domain, we select samples by tokenized length and extract contiguous contexts between 128K and approximately 1M tokens. The microbenchmarks use prefixes of 1M-token inputs; certificate analyses cover samples throughout this length range.

\paragraph{Baselines.}
We compare with SGLang v0.5.18~\citep{sglang} and LiteTopK~\citep{litetopk}. SGLang computes full multihead scores followed by Top-$K$ selection. LiteTopK fuses scoring, candidate filtering, and selection to reduce intermediate memory traffic. \method{} further reduces computation by using certificates to skip unnecessary full-head scoring.

\paragraph{Measurements.}
We measure complete-Indexer latency with CUDA events after warmup, using activations from model runs. End-to-end performance is measured by time to first token (TTFT) for individual requests, clearing the prefix cache between requests.

\subsection{Indexer microbenchmarks}
\label{sec:indexer_microbenchmarks}
\paragraph{Prefill performance.}
We vary context length over 128K, 256K, 512K, 768K, and 1M tokens, and prefill chunk size over 4K, 8K, 16K, and 32K. Each configuration replays the final query chunk across all 21 indexer-producing layers, including query/key preparation in the measured latency.

\method{} is the fastest method in all 60 configurations (Figure~\ref{fig:indexer_generalization}). Speedups over SGLang range from $1.12$ to $2.74\times$, exceeding $2\times$ in every configuration at 512K tokens and above. Against LiteTopK, \method{} achieves $1.17$--$1.69\times$ speedups. At 1M, the gains reach $2.47$--$2.74\times$ and $1.47$--$1.69\times$ over the two baselines, respectively.

\paragraph{Analysis.}
Context length drives the gain more strongly than query chunk size. For English inputs with 16K queries, speedup rises from $1.19\times$ at 128K to $2.58\times$ at 1M. By comparison, increasing the query chunk eightfold at 1M, from 4K to 32K, changes speedup from $2.47\times$ to $2.69\times$. Longer contexts expand the key set evaluated by each query while reusing its certificate preparation. Increasing the chunk size instead scales the query batch for all methods. Thus, \method{} gains most from reducing work along the long key dimension, while remaining effective across prefill batching choices.

The comparison with LiteTopK separates this effect from fusion. Fusion reduces score traffic and selection overhead, but leaves the full multihead computation intact. Certified rejection removes this remaining work, explaining why \method{} continues to pull ahead as contexts grow. Appendix~\ref{app:retrieval_layers} provides the layer-wise retrieval results.

\paragraph{Decoding performance.}
We build document-analysis requests from HELMET, InfiniteBench, and LongBench v2~\citep{helmet,infinitebench,longbench2}, and code-review requests from SQLite. Each workload fills an approximately 1M-token KV cache with equal-length requests. \method{} accelerates the decoding Indexer by $1.06\times$, $1.11\times$, and $1.26\times$ at request lengths of 256K, 512K, and approximately 1M, respectively (Figure~\ref{fig:decoding_indexer}). Here the gain comes from reducing selection work: recently selected keys provide thresholds for compact selection, benefiting decoding even with few queries per step.

\paragraph{Pruning across layers.}
We define the pruning rate as $p=1-N_{\mathrm{keep}}/N_{\mathrm{screen}}$, where $N_{\mathrm{screen}}$ and $N_{\mathrm{keep}}$ count query--key pairs before and after certificate screening. Figure~\ref{fig:indexer_pruning} shows pruning rates above 95\% in 20 of the 21 layers in each domain, with domain averages of 97.2\%--97.5\%. English, Chinese, and code exhibit similar profiles.

The variation follows model depth: all three domains have their lowest pruning rate at layer 1, where 89\%--93\% of pairs are still rejected, and exceed 95\% in every other layer. The shared pattern points to layer-dependent score geometry as a stronger influence than the input domain. In terms of the selection-margin analysis in Appendix~\ref{app:margin}, most keys lie sufficiently far below the Top-$K$ boundary that the certificate can reject them despite its slack. Full-head scoring is then concentrated on the small set of candidates that remain competitive.

\subsection{End-to-end inference performance}
\label{sec:ttft_length}
We measure TTFT on the book and code workloads from 128K to approximately 1M tokens, using TP8 and the model-specific prefill chunks described above.

\method{} nearly halves TTFT at 1M tokens on both models (Figure~\ref{fig:ttft_length}). It achieves $1.86\times$ and $1.94\times$ speedups over SGLang on GLM-5.2 and DeepSeek-V3.2, respectively, and $1.26\times$ and $1.50\times$ over LiteTopK. The gains over LiteTopK show that avoiding score computation improves serving latency even after scoring and selection have been fused.

\paragraph{Analysis.}
The growing Indexer share explains why the end-to-end gain strengthens with context length. With Indexer time fraction $f$ and Indexer speedup $s$, Amdahl's law gives an overall speedup of $S=(1-f+f/s)^{-1}$. For GLM-5.2, indexing grows from 23.6\% of prefill GPU time at 128K to 73.0\% at 1M; DeepSeek-V3.2 reaches 90.8\% at 1M (Figure~\ref{fig:indexer_cost}). As this fraction increases, accelerating indexing removes a larger share of request latency. Attention and feed-forward layers continue to contribute to request latency, explaining why TTFT speedup is smaller than Indexer speedup. This explains both the increasing gains in Figure~\ref{fig:ttft_length} and the value of targeting indexing for long-context serving.

\subsection{Ablation study}
We enable screening and sparse writeback independently: Base uses neither, Full uses sparse writeback alone, Dense uses screening alone, and \method{} combines both. The workload is a SciFact retrieval example from LOFT~\citep{loft}, with 16K queries and 896K keys; we average retrieval latency across all 21 indexer-producing layers.

Screening alone reduces latency by 37.4\%, and sparse writeback alone by 23.4\% (Figure~\ref{fig:ablation}). Together, they reduce retrieval latency by 47.3\%, yielding a $1.90\times$ speedup over Base. Adding screening to Full saves a further 31.1\%, while adding sparse writeback to Dense saves 15.7\%.

\paragraph{Analysis.}
Full avoids dense score writeback but still pays for every head evaluation; its remaining gap to \method{} exposes the computation that fusion alone cannot remove. Dense makes the opposite tradeoff: screening reduces score evaluation, while dense intermediate results continue to consume HBM bandwidth. The additional gain from sparse writeback therefore shows why pruning needs a matching execution strategy to realize its benefit.

The certificate-component ablation further identifies what makes pruning effective. Removing the negative aggregate raises candidate survival from 2.65\% to 95.06\% (Appendix~\ref{sec:certificate_ablation}), nearly eliminating rejection. The negative contribution is therefore essential to obtaining a tight certificate for the signed Indexer score. Appendix~\ref{sec:certificate_layers} shows how this effect varies across layers.

\FloatBarrier
\setcounter{topnumber}{2}

\section{Related work}
\label{sec:related}
\paragraph{Sparse attention with lightning indexers.}
Sparse attention restricts query--key interactions. NSA combines token compression with fine-grained selection \citep{nsa2025}; MoBA uses expert routing for block attention \citep{moba2025}; DSA selects tokens with a learned multihead indexer \citep{deepseek32}. MiniMax-M3 combines block attention with learned indexing \citep{minimaxmsa}. IndexCache reuses indices across layers \citep{indexcache,glm53,hy4preview}, while SparDA predicts next-layer selections \citep{sparda}. HISA uses hierarchical block indexing and token refinement for token-level attention \citep{hisa}; DeepSeek-V4.1-Flash and LongCat-2.0 combine hierarchical indexing with cross-layer reuse \citep{deepseek41flash,longcat2}. PIVOT shares proxy scans across query groups \citep{pivot}. MISA selects indexer heads with optional full-score reranking \citep{misa}; Qwen3.8-Flash-Next uses only four query heads in its indexer \citep{qwen38next}. These approaches need not preserve the original full scorer's Top-$K$, and reranking cannot recover excluded tokens. \method{} certifies rejection under the original scorer.

\paragraph{Exact GPU selection and fused scoring.}
Radix selection narrows Top-$K$ candidates through successive digit histograms, without fully sorting the scores \citep{zhang2023topk}. GVR uses temporal correlation to predict a decoding threshold, verifies it by counting materialized scores, and refines candidates in shared memory \citep{gvr}. DeepSelect streams materialized score rows through a bounded candidate buffer with a monotonically increasing threshold \citep{deepselect}. LiteTopK fuses full-head scoring with online candidate filtering, reducing score writeback and selection input \citep{litetopk}. These methods accelerate Top-$K$ selection but still rely on full-head scores for every key in DSA. \method{} builds on streaming thresholds and compact selection by adding a score upper bound that rejects keys \emph{before} full-head evaluation.

\paragraph{Score bounds in classical document retrieval.}
MaxScore \citep{maxscore1995} and WAND \citep{wand2003} use term-contribution upper bounds and a running Top-$K$ threshold to skip unnecessary document scoring. Block-Max WAND tightens these bounds for individual posting-list blocks \citep{blockmaxwand2011}. \method{} applies the same pruning principle to the signed multihead scores of DSA. A block-level upper bound must cover every key in the block and can therefore be loose for individual candidates. This motivates tighter, per-key bounds that exploit query-head geometry and activation structure to reject candidates before full-head scoring.

\section{Limitations and conclusion}
\method{} shows that exact token selection can require less work than computing every token's full score. Signed-head certificates, original-score completion, and sparse writeback turn this observation into conditional computation within the attention indexer. Query geometry determines how much work can be certified unnecessary; the head-budget analysis makes the cost of obtaining that certificate explicit.

The TTFT results show increasing benefits at long contexts, while the workload analysis explains the tradeoff between screening cost and avoided full scoring. \method{} still scans every key, and short contexts or loose bounds can leave insufficient savings to offset screening. Numerical guarantees require conservative evaluation relative to the original scorer; valid tied Top-$K$ sets need not yield identical generated text. The measurements establish computational benefits, not improved task accuracy beyond a model's trained context.

\paragraph{Future work.}
Training-time regularization of certificate slack could encourage an indexer whose computation is easier to skip safely. Following upper-bound pruning in document retrieval, future work could extend our per-key certificates to key-block bounds using compact summaries. Blocks whose bounds fall below the selection threshold could then be skipped without reading individual key vectors. These directions could reduce both scoring and key-access costs as context grows.

\label{lastmainpage}
\clearpage
\section*{AI use statement}
We used generative AI tools for manuscript drafting and revision and for drafting initial GPU kernel implementations, which were subsequently revised by the authors. The authors take responsibility for the final manuscript and implementation.

\section*{Reproducibility Statement}
Sections~\ref{sec:method} and~\ref{sec:implementation} describe the method and GPU implementation. Appendix~\ref{app:proofs} proves certificate validity and exact selection; Appendix~\ref{app:tightness} analyzes tightness and computational cost; Appendix~\ref{app:numerical} gives numerical conditions and recovery requirements. Section~\ref{sec:experiments} specifies the workloads and metrics, and Appendix~\ref{app:experimental_details} details indexer profiling, certificate ablations, and layer-wise retrieval measurements. Accompanying CPU checks verify algebraic relations and fixed-threshold selection.

\clearpage
\bibliography{iclr2027_conference}
\bibliographystyle{iclr2027_conference}
\clearpage
\appendix
\section{Certificate Validity and Exact Selection}
\label{app:proofs}
We establish the score bounds in real arithmetic and prove exact selection for a fixed query over its visible keys. Numerical conditions are given in Appendix~\ref{app:numerical}.

\subsection{Validity of the score certificate}
\label{app:certificate_validity}
\label{app:grouped}
\label{app:positive_joint}
\begin{proposition}[Positive-head projection bound]
For any $\beta_h\ge0$ and decomposition $\q_h=\beta_h\q_j+\mathbf d_h$,
\[
 [\q_h^\top\kv]_+\le\beta_h[\q_j^\top\kv]_++\norm{\mathbf d_h}\norm{\kv}.
\]
\end{proposition}
\begin{proof}
For real $a,b$ and $c\ge0$, $[a+b]_+\le[a]_++[b]_+$ and $[ca]_+=c[a]_+$. Hence
\[
 [\beta_h\q_j^\top\kv+\mathbf d_h^\top\kv]_+
 \le\beta_h[\q_j^\top\kv]_++[\mathbf d_h^\top\kv]_+
 \le\beta_h[\q_j^\top\kv]_++\norm{\mathbf d_h}\norm{\kv}.
\]
Multiplying by $w_h>0$ and summing over $P$ gives the positive part of Equation~\eqref{eq:bound}. Any finite nonnegative coefficient is valid if its residual is bounded conservatively.
\end{proof}

\begin{proposition}[Negative-head support bound]
For any $A\subseteq N$, $[\mathbf a_A^\top\kv]_+\le\sum_{h\in N}|w_h|[\q_h^\top\kv]_+$.
\end{proposition}
\begin{proof}
For $c_h\ge0$, ReLU subadditivity gives $[\sum_{h\in A}c_hx_h]_+\le\sum_{h\in A}c_h[x_h]_+\le\sum_{h\in N}c_h[x_h]_+$. Substitute $c_h=|w_h|$ and $x_h=\q_h^\top\kv$. The result holds pointwise for any data-dependent support, including $A=\varnothing$, when the left-hand side is zero.
\end{proof}

Together, these bounds establish $s_i\le U_i$ in Equation~\eqref{eq:bound}. The following grouped and joint bounds can tighten its positive residual allowance.

\paragraph{Grouped positive residuals.}
Partition the coordinates into $G$ disjoint groups. In real arithmetic, let
$E_{P,g}=\sum_{h\in P}w_h\norm{\mathbf d_h^{(g)}}$. Then
\begin{align}
 \sum_{h\in P}w_h[\mathbf d_h^\top\kv_i]_+
 &\le\sum_{g=1}^G E_{P,g}\norm{\kv_i^{(g)}}\\
 &\le\left(\sum_gE_{P,g}^2\right)^{1/2}\norm{\kv_i}.
 \label{eq:groupbound}
\end{align}
Minkowski's inequality gives
\[
 \left(\sum_g E_{P,g}^2\right)^{1/2}
 =\left\lVert\sum_{h\in P}w_h(\norm{\mathbf d_h^{(1)}},\ldots,\norm{\mathbf d_h^{(G)}})\right\rVert_2
 \le\sum_{h\in P}w_h\norm{\mathbf d_h}=E_P.
\]
This tightens the residual coefficient with one RMS value per key.

\paragraph{Joint positive residuals.}
The grouped bound can also exploit dependence among residual directions. For one coordinate group, write $\mathbf u_h=w_h\mathbf d_h^{(g)}$ and define
\begin{equation}
 B_g=\tfrac12\left(
 \left\lVert\sum_{h\in P}\mathbf u_h\right\rVert_2+
 \sqrt{\sum_{h,\ell\in P}|\mathbf u_h^\top\mathbf u_\ell|}\right).
 \label{eq:group_joint}
\end{equation}
\begin{proposition}[Joint residual allowance]
For every $\mathbf k^{(g)}$,
$\sum_{h\in P}[\mathbf u_h^\top\mathbf k^{(g)}]_+
\le B_g\norm{\mathbf k^{(g)}}$, and
$B_g\le\sum_{h\in P}\norm{\mathbf u_h}=E_{P,g}$.
\end{proposition}
\begin{proof}
Using $[x]_+=(x+|x|)/2$ gives
\[
 \sum_h[\mathbf u_h^\top\mathbf k^{(g)}]_+
 =\tfrac12\left[
   \left(\sum_h\mathbf u_h\right)^\top\mathbf k^{(g)}
   +\sum_h|\mathbf u_h^\top\mathbf k^{(g)}|\right].
\]
The absolute-value sum equals
$\max_{\sigma_h\in\{-1,1\}}(\sum_h\sigma_h\mathbf u_h)^\top\mathbf k^{(g)}$.
For every sign vector,
$\norm{\sum_h\sigma_h\mathbf u_h}^2
\le\sum_{h,\ell}|\mathbf u_h^\top\mathbf u_\ell|$.
Cauchy--Schwarz yields the first claim. The triangle inequality bounds the center norm by $\sum_h\norm{\mathbf u_h}$, and
$\sum_{h,\ell}|\mathbf u_h^\top\mathbf u_\ell|
\le(\sum_h\norm{\mathbf u_h})^2$ proves the second.
\end{proof}
ReLU subadditivity and Cauchy--Schwarz combine groups into $\sqrt{\sum_g B_g^2}\norm{\kv}$. A single group proves Equation~\eqref{eq:positive_joint}. The bound needs no additional key summary, but its query-side cost is quadratic in the number of omitted positive heads.

\subsection{Fixed-threshold algorithm and exact Top-\texorpdfstring{$K$}{K} selection}
\label{app:exact_selection}
Assume finite scores and $|\Omega|\ge K$. Top-$K$ returns $K$ distinct token IDs; identical tied IDs require a common total order. A seed set of distinct visible keys supplies a threshold that remains fixed during screening and refinement.

\begin{proposition}[Subset threshold]
For any $\mathcal S\subseteq\Omega$ with $|\mathcal S|\ge K$, $\tau_K(\mathcal S)\le\tau_K(\Omega)$.
\end{proposition}
\begin{proof}
At least $K$ distinct elements of $\mathcal S$, hence of $\Omega$, have score at least $\tau_K(\mathcal S)$. The full-domain $K$th score can therefore be no smaller.
\end{proof}

\begin{figure}[htbp]
\small
\hrule\vspace{5pt}
\refstepcounter{algorithm}\textbf{Algorithm \thealgorithm}\quad \method{} with a fixed seed threshold\label{alg:dosa}\par
\vspace{3pt}
\algline{1}{Choose distinct visible seed keys $\mathcal S\subseteq\Omega$, $|\mathcal S|\ge K$.}
\algline{2}{Score $\mathcal S$ with the original scorer; $\tau\gets\tau_K(\mathcal S)$.}
\algline{3}{Build $R$, anchors, $A$, and the score bound in Equation~\eqref{eq:bound}.}
\algline{4}{$V\gets\{i\in\Omega\setminus\mathcal S:U_i\ge\tau\}$.}
\algline{5}{Score every $i\in V$ using the original $H$-head scorer.}
\algline{6}{$C\gets\{i\in V:s_i\ge\tau\}$.}
\algline{7}{Return $\topk(\mathcal S\cup C)$.}
\vspace{4pt}\hrule
\smallskip{\footnotesize Both filters retain equality. Final selection uses original scores and the common tie order when identical tied IDs are required.}
\end{figure}

\begin{theorem}[Exact Top-$K$ selection]
\label{thm:exact}
In Algorithm~\ref{alg:dosa}, let $\tau=\tau_K(\mathcal S)$ come from at least $K$ distinct visible keys scored by the original scorer. Suppose every key rejected during screening satisfies $s_i\le U_i<\tau$, and no live candidate is lost. With a common total tie order, the algorithm returns the same Top-$K$ as full-domain scoring. Without a prescribed tie order, it returns a valid Top-$K$ set under Equation~\eqref{eq:score}.
\end{theorem}
\begin{proof}[Proof of Theorem~\ref{thm:exact}]
Consider any key $i\notin\mathcal S\cup C$. If screening rejects it, then $s_i\le U_i<\tau$. Otherwise it belongs to $V\setminus C$ and its original score satisfies $s_i<\tau$. In both cases, the subset-threshold proposition gives
\[
 s_i<\tau_K(\mathcal S)\le\tau_K(\Omega).
\]
Thus no discarded key can belong to the full-domain Top-$K$, including under any tie order. Since $\mathcal S\cup C$ contains at least $K$ distinct keys and every possible winner, applying the same total order to its original scores returns exactly $\topk(\Omega)$. Without a prescribed tie order, selection includes every score strictly above the global $K$th score and fills the remaining places with distinct keys at that score. Identical tied IDs are not required in that case.
\end{proof}
Candidate-storage and recovery requirements that enforce the no-loss condition are given in Appendix~\ref{app:recovery}.

\section{Certificate Tightness and Computational Cost}
\label{app:tightness}
We analyze bound slack, selection margins, and screening cost in finite real arithmetic. These determine efficiency separately from correctness.

\subsection{Minimum uniform allowance for a single anchor}
\label{app:projection_optimal}
\begin{proposition}[Optimal single-anchor allowance]
\label{prop:projection_optimal}
For vectors $\q,\mathbf p\in\RR^D$, consider all pairs $\beta,c\ge0$ satisfying
\begin{equation}
 [\q^\top\kv]_+\le\beta[\mathbf p^\top\kv]_++c\norm{\kv}
 \quad\text{for every }\kv\in\RR^D.
 \label{eq:app_anchor_class}
\end{equation}
The smallest achievable coefficient is $c^*=\min_{\beta\ge0}\norm{\q-\beta\mathbf p}$, the distance from $\q$ to the nonnegative ray generated by $\mathbf p$. For $\mathbf p\ne0$, it is attained by
\[
 \beta^*=\frac{[\q^\top\mathbf p]_+}{\norm{\mathbf p}^2},\qquad
 c^*=\norm{\q-\beta^*\mathbf p}.
\]
For $\mathbf p=0$, take $\beta^*=0$ and $c^*=\norm{\q}$.
\end{proposition}

\begin{proof}[Proof of Proposition~\ref{prop:projection_optimal}]
The positive-head projection bound already establishes feasibility of $(\beta^*,c^*)$. We prove that no choice of $\beta$ permits a smaller $c$. The case $\q=0$ is immediate. If $\q^\top\mathbf p\le0$, substitute $\kv=\q/\norm{\q}$ into Equation~\eqref{eq:app_anchor_class}. Its anchor term vanishes for every $\beta\ge0$, giving $c\ge\norm{\q}=c^*$. This includes a zero anchor. Otherwise let $\mathbf d=\q-\beta^*\mathbf p$. The projection satisfies $\mathbf d^\top\mathbf p=0$. When $\mathbf d\ne0$, substituting $\kv=\mathbf d/\norm{\mathbf d}$ again makes the anchor term vanish, while $[\q^\top\kv]_+=\norm{\mathbf d}$. Thus $c\ge c^*$ for every $\beta$. If $\mathbf d=0$, feasibility with $c=0$ attains the trivial lower bound.
\end{proof}

Minimizing $c^*$ over retained anchors gives Equation~\eqref{eq:projection}. This optimizes the uniform norm allowance, not the complete bound pointwise or the retained-head set. Rounded coefficients remain valid when their actual residuals are conservatively bounded.

\subsection{Certificate slack and activation-aware head selection}
\label{app:slack}
\label{app:support_surrogate}
For one key, abbreviate $x_h=\q_h^\top\kv$ and $c_h=|w_h|$. Let
\begin{align*}
 B_h(\kv)&=\beta_h[\q_{j_h^\star}^\top\kv]_++\norm{\mathbf d_h}\norm{\kv},\\
 \delta_P(\kv)&=\sum_{h\in P}w_h\bigl(B_h(\kv)-[x_h]_+\bigr),\\
 Z_A^+(\kv)&=\sum_{h\in A}c_h[x_h]_+,\qquad
 Z_A^-(\kv)=\sum_{h\in A}c_h[-x_h]_+.
\end{align*}

\begin{proposition}[Exact certificate slack]
\label{prop:slack}
With the notation above, the bound in Equation~\eqref{eq:bound} has three nonnegative slack components:
\begin{equation}
 \boxed{U(\kv)-s(\kv)=
 \underbrace{\delta_P(\kv)}_{\text{positive projection}}
 +\underbrace{\sum_{h\in N\setminus A}c_h[x_h]_+}_{\text{missed active negatives}}
 +\underbrace{\min\{Z_A^+(\kv),Z_A^-(\kv)\}}_{\text{aggregate cancellation}}.}
 \label{eq:slack}
\end{equation}
\end{proposition}

\begin{proof}[Proof of Proposition~\ref{prop:slack}]
The retained contributions cancel in $U-s$. The positive remainder is $\delta_P\ge0$ by Equation~\eqref{eq:positive}. The negative remainder is
\begin{align*}
 \delta_N(\kv)
 &=\sum_{h\in N}c_h[x_h]_+-[Z_A^+-Z_A^-]_+\\
 &=\sum_{h\in N\setminus A}c_h[x_h]_+
   +\min\{Z_A^+,Z_A^-\},
\end{align*}
where the last equality uses $z-[z-y]_+=\min(z,y)$ for $z,y\ge0$. This proves Equation~\eqref{eq:slack} and nonnegativity of all its components.
\end{proof}

Let $M(\kv)=\{h\in N:x_h>0\}$ be the active omitted negative heads. Choosing $A=M(\kv)$ makes both negative slack terms zero; including inactive heads can instead introduce cancellation. The positive headwise decomposition above uses the triangle allowance $E_P$. Replacing it with a valid joint allowance preserves nonnegative slack but changes its positive term.

\paragraph{Support and retained-head selection.}
For any support, the negative slack obeys
\begin{align}
 0\le\delta_N(\kv)
 &\le\sum_{h\in A\triangle M(\kv)}c_h|\q_h^\top\kv|\nonumber\\
 &\le\norm{\kv}\sum_{h\in A\triangle M(\kv)}a_h,
 \qquad a_h=c_h\norm{\q_h},
 \label{eq:app_disagreement}
\end{align}
where $\triangle$ denotes symmetric difference. The first inequality bounds aggregate cancellation by $Z_A^-$; the second is Cauchy--Schwarz.

For nonzero sample keys, let $p_h$ be the observed activation frequency. Averaging the last expression in Equation~\eqref{eq:app_disagreement} yields the separable surrogate
\begin{equation}
 J(A)=\sum_{h\in A}a_h(1-p_h)+\sum_{h\in N\setminus A}a_hp_h.
 \label{eq:app_surrogate}
\end{equation}
If the masks are the signs of the same real dot products, $J(A)$ upper-bounds the sample mean of $\delta_N(\kv)/\norm{\kv}$. Each head independently contributes either $a_h(1-p_h)$ or $a_hp_h$, so componentwise majority, $A^*=\{h:p_h\ge1/2\}$, minimizes $J$. The minimum is $\sum_h a_h\min(p_h,1-p_h)$.

If all positive heads are retained, evaluating $b$ additional negative heads removes their surrogate terms. Retaining the $b$ largest $a_h\min(p_h,1-p_h)$ therefore minimizes the remaining surrogate jointly with majority support. This optimizes the sampled disagreement bound, not survivor count or latency. Rounded activation signs near zero can affect the surrogate interpretation, but any resulting support remains safe.

\subsection{Selection margins and screening cost}
\label{app:margin}
\label{app:budget_marginal}
Let $\tau^*$ be the full-domain $K$th score, $\Delta_i=\tau^*-s_i$ a key's selection margin, $b=\tau^*-\tau\ge0$ the threshold deficit, and $\delta_i=U_i-s_i\ge0$ the bound slack. Then
\begin{equation}
 \boxed{U_i<\tau
 \quad\Longleftrightarrow\quad
 \Delta_i>\delta_i+b.}
 \label{eq:margin}
\end{equation}
A stronger seed reduces $b$; a tighter certificate reduces $\delta_i$. If $\delta_i\le\epsilon$, survivors lie within $\{i:s_i\ge\tau^*-\epsilon-b\}$. This boundary mass need not be small, and tied keys survive automatically. The same relation holds for computed scores $\sref_i$ and valid numerical bounds $U_i^{\rm num}$; keys without valid bounds are retained separately.

\paragraph{Arithmetic cost and head budget.}
The survivor count determines whether screening saves arithmetic. Let $N_0$ keys be fully scored during initialization, $N_s$ be screened, and $M$ survive screening. With $d_a$ auxiliary rows, logical dot-product counts are approximately
\begin{align}
 D_{\mathrm{base}}&=H(N_0+N_s),\\
 D_{\method}&=HN_0+(r+d_a)N_s+HM.
 \label{eq:work}
\end{align}
Writing $f=N_0/(N_0+N_s)$ and $m=M/N_s$, normalized work is
\begin{equation}
 \frac{D_{\method}}{D_{\mathrm{base}}}
 \approx f+(1-f)\left(\frac{r+d_a}{H}+m\right).
 \label{eq:ratio}
\end{equation}
Screening saves logical dot products only when $m<1-(r+d_a)/H$. More precisely, let $p(r)=1-m(r)$ be rejection at budget $r$. Within this logical-work model, for fixed initialization and auxiliary-row count,
\begin{equation}
 D(r+\Delta r)<D(r)
 \quad\Longleftrightarrow\quad
 p(r+\Delta r)-p(r)>\frac{\Delta r}{H}.
 \label{eq:budget_marginal}
\end{equation}
An additional head pays off only if it avoids enough full rescoring. If the auxiliary-row count changes, replace $\Delta r$ by $\Delta r+\Delta d_a$; changed anchors or supports need not improve rejection monotonically. This model excludes planning, key summaries, Tensor Core padding, selection, and recovery. Every key is still scanned, and screening adds work if all keys survive.

\section{Numerical Conditions and Recovery}
\label{app:numerical}
This section gives sufficient conditions for preserving exact selection under the original scorer's arithmetic, together with candidate-storage and recovery requirements.

\subsection{Reference scores and the FP8 certificate}
\label{sec:numerical}
Let $\sref_i$ be the original indexer's computed score. Numerical dominance $\sref_i\le U_i^{\rm num}$ must be established separately from $s_i\le U_i$. The tie-order requirements of Theorem~\ref{thm:exact} also apply to computed scores, and equality survives pruning.

Theorem~\ref{thm:exact} applies with $\sref_i$ and $U_i^{\rm num}$ when the seed threshold and candidate scores use the same original scorer and every key rejected during screening satisfies $\sref_i\le U_i^{\rm num}<\tau$. Failure handling is specified in Appendix~\ref{app:recovery}.

Write an FP8 key as $\kv_i=\alpha_i\bar{\kv}_i$, where $\alpha_i>0$ and $\bar{\kv}_i$ is its $D$-dimensional payload. In the base form of Equation~\eqref{eq:bound}, two auxiliary rows approximate the negative aggregate:
\begin{equation}
 \mathbf a_A=\mathbf v_0+\mathbf v_1+\mathbf e_A.
 \label{eq:dd8}
\end{equation}
Their dot products are added before ReLU. With stored folded weights $\omega_j$, the base screening expression in payload units is
\begin{equation}
 \widehat U_i=\operatorname{fl}_{\rm screen}\!\left(
 \sum_{j\in R}\omega_j[\q_j^\top\bar{\kv}_i]_+
 -[(\mathbf v_0+\mathbf v_1)^\top\bar{\kv}_i]_+\right).
 \label{eq:screen_program}
\end{equation}
\paragraph{Absolute-value form.}
When every positive-weight head is retained, $P=\varnothing$. Set $\mathbf V=\mathbf a_A-\tfrac12\sum_{j\in R}w_j\q_j$. Using $[x]_+=(x+|x|)/2$,
\begin{align}
 U_{\rm abs}(\kv)
 &=\tfrac12\sum_{j\in R}w_j|\q_j^\top\kv|-\mathbf V^\top\kv\nonumber\\
 &=S_R(\kv)-\mathbf a_A^\top\kv
 \ge S_R(\kv)-[\mathbf a_A^\top\kv]_+
 \ge s(\kv).
 \label{eq:abs_bound}
\end{align}
The additional slack relative to Equation~\eqref{eq:bound} is $[-\mathbf a_A^\top\kv]_+$. The 32-head implementation uses this form for $r=14,16$; the $r=20$ path retains the base form with positive projection. Encoding $\mathbf V=\mathbf v_0+\mathbf v_1+\mathbf e_V$ with two scaled FP8 rows gives
\begin{equation}
 \widehat U_i=\operatorname{fl}_{\rm screen}\!\left(
 \tfrac12\sum_{j\in R}w_j|\q_j^\top\bar{\kv}_i|
 -(\mathbf v_0+\mathbf v_1)^\top\bar{\kv}_i\right).
 \label{eq:abs_screen_program}
\end{equation}
Both forms use $r+2$ logical dot products; survivors receive the original $H$-head scores. The numerical allowances cover auxiliary-vector representation, weight formation, and arithmetic errors, together with any omitted positive residuals.

Let $\rho_i\ge\norm{\bar{\kv}_i}/\sqrt D$ be a conservative key RMS summary. A numerical certificate has the form
\begin{equation}
 U_i^{\mathrm{num}}=
 \operatorname{Eval}_{\rm up}\!\left[
 \alpha_i\bigl(\widehat U_i+\rho_i C+\varepsilon_{\mathrm{raw}}\bigr)
 +\varepsilon_{\mathrm{out}}\right].
 \label{eq:fp8}
\end{equation}
The nonnegative allowances $C$, $\varepsilon_{\rm raw}$, and $\varepsilon_{\rm out}$ satisfy Appendix~\ref{app:fp8_accounting}. The operator $\operatorname{Eval}_{\rm up}$ encloses the expression upward through directed rounding or explicit error bounds; round-to-nearest alone is insufficient. These conditions ensure $\sref_i\le U_i^{\rm num}$ for either screening expression.

\subsection{Error allowances and conservative key summaries}
\label{app:fp8_accounting}
Let $E_g^+$ be an upper enclosure of a valid positive residual allowance for group $g$, using the grouped triangle or joint bound in Appendix~\ref{app:certificate_validity}. Residual formation and group norms must be enclosed upward; joint bounds additionally require conservative weighted residuals, sums, Gram entries, and final norms. Suppose a nonnegative coordinate envelope $\mathbf J$ and absolute allowances satisfy
\begin{equation}
 \sref_i\le\alpha_i\!\left(
 \widehat U_i+\sum_g E_g^+\norm{\bar{\kv}_i^{(g)}}
 +\sum_d J_d|\bar k_{i,d}|+\varepsilon_{\rm raw}\right)
 +\varepsilon_{\rm out}.
 \label{eq:fp8_error_contract}
\end{equation}
The envelope covers auxiliary-vector reconstruction, weight formation, screening arithmetic, and the original scorer's arithmetic. In the base form, aggregate error contributes $|e_{A,d}|$ and folded-weight errors contribute $\sum_{j\in R}|\widetilde w_j-\omega_j||q_{j,d}|$ to $J_d$, by ReLU's Lipschitz property. In the absolute-value form, $E_g^+=0$ and the auxiliary-vector contribution is $|e_{V,d}|$; the absolute-value and linear terms admit the same coordinate-envelope treatment. Underflow and other absolute errors enter $\varepsilon_{\rm raw}$ or $\varepsilon_{\rm out}$ according to whether they occur before or after key scaling.

Applying Cauchy--Schwarz within groups and then across groups shows that a sufficient RMS coefficient is any conservative enclosure
\begin{equation}
 C\ge\sqrt D\left(\sum_g\bigl[E_g^++\norm{\mathbf J^{(g)}}\bigr]^2\right)^{1/2}.
 \label{eq:fp8_coefficient}
\end{equation}
Together with conservative $\rho_i$, this establishes Equation~\eqref{eq:fp8}. In the absolute-value paths, directed FP64 intervals enclose $\mathbf V$ and its reconstruction residual. The planner combines these with screening and reference-arithmetic allowances and rounds $C$ outward to FP32; the kernels evaluate the correction and positive key scaling with upward rounding.

\paragraph{Conservative key RMS.}
The stored summary must satisfy $\rho_i\sqrt D\ge\norm{\bar{\kv}_i}$. This requires a conservative norm computation and upward coverage in the stored summary format; a nearest-rounded RMS is insufficient. Nonfinite payloads or a failure to establish coverage invalidate the certificate.

\subsection{Candidate storage and full-domain recovery}
\label{app:recovery}
The final output contains $K$ entries, but candidates awaiting full scoring or selection may be far more numerous. A buffer capacity $C_{\rm buf}$ does not limit how many keys satisfy $U_i\ge\tau$: if all scores equal $\tau$, every key passes even with exact bounds. A full buffer must be processed and reduced through exact Top-$K$ merging of original scores, extended, or handled by full-domain recovery. Dropping overflow candidates can lose required Top-$K$ IDs. Chunked processing must preserve the selection and tie-order contract of Algorithm~\ref{alg:dosa}.

Invalid scales, RMS summaries, or nonfinite bounds prevent pruning. Query-level errors remain flagged until full-domain recovery, which evaluates the original scorer and selector over all visible keys, including those discarded before error detection. Re-selecting from an incomplete buffer is insufficient. Exceptional values follow the backend's policy; the finite-score theorem does not define NaN semantics.

\section{Additional Experimental Analysis}
\label{app:experimental_details}
\subsection{Indexer profiling methodology}
\label{app:intro_profile}
Figure~\ref{fig:indexer_cost} uses real model weights and text with DoSA and LiteTopK disabled. Each model--length pair contains two independent requests with a cleared prefix cache. Time is the mean over requests and tensor-parallel ranks of Nsight Systems GPU intervals, cross-checked with CUDA events. The indexer interval includes query/key preparation, key-cache access, scoring, and selection; the denominator is the complete model prefill GPU interval. These quantities exclude CPU tokenization and request transport and are distinct from TTFT.

GLM-5.2 uses 32K prefill chunks and DeepSeek-V3.2 uses 16K chunks, both with TP8. FLOPs count causal logical matrix work, including active experts and sparse attention, with an FMA counted as two operations and no duplication for tensor parallelism. Top-$K$ comparisons, elementwise operations, and padding are not matrix FLOPs. Here 1K denotes 1,024 tokens, 375K denotes 384,000 tokens, and approximately 1M denotes 1,048,569 tokens. DeepSeek measurements beyond its declared context support a computational scaling observation, not a claim about task accuracy.

\subsection{Certificate component ablation}
\label{sec:certificate_ablation}
We isolate the contributions of the certificate components on the per-domain sample sets described in Section~\ref{sec:experiments}. For each of the 21 indexer-producing layers of GLM-5.2, we replay the last 4K queries with the same initial candidate pool. We compare \method{} with four variants: aggregating all omitted negative heads, removing the negative aggregate, selecting retained negative heads by strength alone, and removing positive-head projection. Each variant is timed ten times per input and layer.

\begin{figure}[htbp]
\centering
\includegraphics[width=\linewidth]{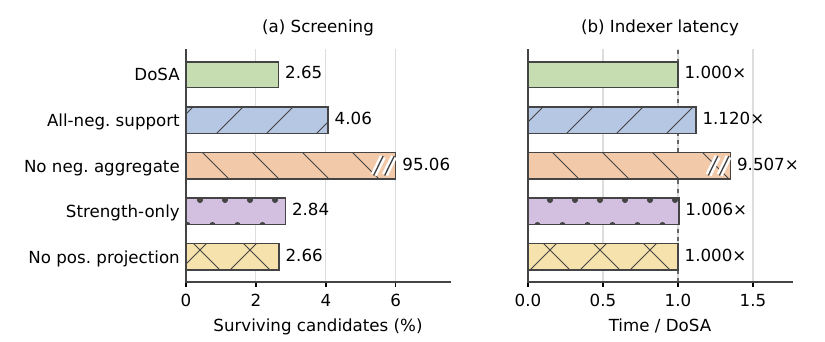}
\caption{Effects of certificate components on screening and Indexer latency. Bars show the mean candidate survival rate and geometric mean latency relative to \method{}. Breaks mark bars beyond the displayed range; annotations give their measured values.}
\label{fig:certificate_mechanism}
\end{figure}

Figure~\ref{fig:certificate_mechanism} shows that removing negative aggregation raises candidate survival from 2.65\% to 95.06\% and increases Indexer time by $9.51\times$. These timings include recovery when the candidate pool fills. Aggregating all negative heads also weakens filtering: survival rises to 4.06\% and latency increases by 12.0\%, supporting the use of activation-informed support. Strength-only head selection and removal of positive-head projection have much smaller aggregate effects in this workload. Positive-head projection is required by only 0.42\% of the evaluated query-layer rows. Appendix~\ref{sec:certificate_layers} breaks down these effects by layer.

\subsection{Layer-wise certificate behavior}
\label{sec:certificate_layers}
Figure~\ref{fig:certificate_layers} decomposes the experiment in Appendix~\ref{sec:certificate_ablation} by layer, aggregating over the same per-domain sample sets. Removing the negative aggregate leaves nearly all candidates alive in most layers. The cost of aggregating all negative heads is most pronounced in early layers: at layer 2, its relative Indexer latency reaches $3.41\times$.

\begin{figure}[htbp]
\centering
\includegraphics[width=\linewidth]{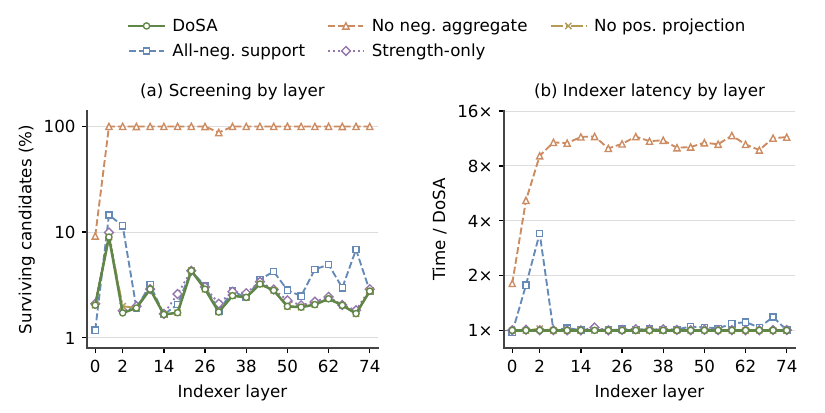}
\caption{Layer-wise effects of certificate components on candidate survival and Indexer latency. Each layer aggregates the same sample sets. Both vertical axes use logarithmic scales.}
\label{fig:certificate_layers}
\end{figure}

\newpage
\subsection{Layer-wise retrieval speedups}
\label{app:retrieval_layers}
For the GLM-5.2 sample sets used in the certificate analysis, we replay the final 32K queries across all 21 indexer-producing layers (252 cells). Figures~\ref{fig:api_layers_sglang} and~\ref{fig:api_layers_lite} divide baseline median latency by \method{} median latency over 12 replays per method and cell. Timings cover scoring and selection on prepared queries and keys, excluding query/key preparation. The heatmaps resolve variation across inputs and layers.

\paragraph{Measurement scope.}
The matrix retains 247 cells from the original run and replaces five cells with matched measurements of all three methods after a candidate-storage capacity repair. The remeasured cells are EN01 at layer 26 and EN02, EN04, ZH04, and CD04 at layer 1. The repair preserves the scoring and selection rules, four-anchor initialization, and 16K internal query scratch. At approximately 1M-key capacity, it reserves 1,024 candidate-ID pages per query and 114,688 output records, adding 9.94\,GiB of workspace relative to the original allocation. The full matrix was not rerun after the repair.

Within these measured cells, \method{} is faster in 251 of 252 comparisons (99.6\%) against SGLang and 250 of 252 (99.2\%) against LiteTopK. All 21 layers achieve at least $2\times$ speedup over SGLang on one or more inputs, with gains reaching $3.74\times$; speedups over LiteTopK reach $2.08\times$. The three slower comparisons add 0.7--9.0\% latency and are retained in the figures. Across the input samples, every layer has a median speedup above $1\times$ against both baselines.

\begin{figure}[p]
\centering
\includegraphics[width=\linewidth]{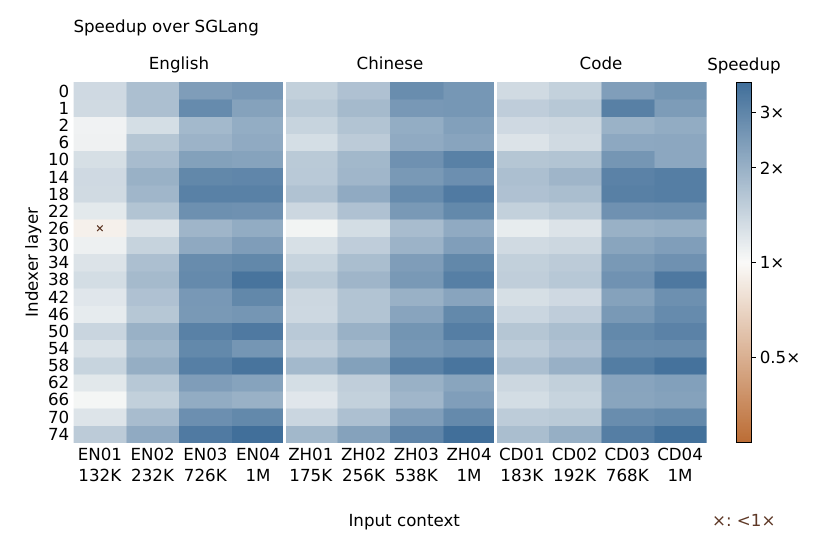}
\caption{Per-layer retrieval speedup over SGLang for the final 32K queries. Columns identify inputs and context lengths; rows identify indexer layers. Crosses mark cells slower than the baseline. The logarithmic color scale is centered at $1\times$. Five cells use the capacity-repair remeasurements specified in Appendix~\ref{app:retrieval_layers}.}
\label{fig:api_layers_sglang}
\medskip
\includegraphics[width=\linewidth]{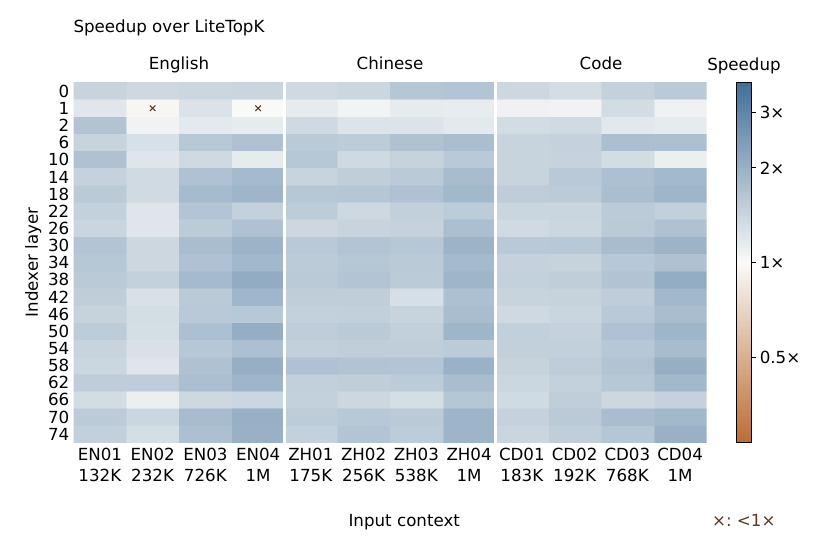}
\caption{Per-layer retrieval speedup over LiteTopK. Input and layer ordering, timing scope, remeasured cells, and color scale match Figure~\ref{fig:api_layers_sglang}; crosses retain the two comparisons slower than the baseline.}
\label{fig:api_layers_lite}
\end{figure}

\end{document}